\documentclass[10pt]{article}
\usepackage[utf8]{inputenc}
\usepackage{geometry}
\usepackage{bbm}
\usepackage{amsmath,amssymb,amsthm}
\theoremstyle{definition}
\newtheorem{theorem}{Theorem}[section]
\newtheorem{proposition}[theorem]{Proposition}

\newtheorem{corollary}[theorem]{Corollary}
\newtheorem{lemma}[theorem]{Lemma}

\newtheorem{remark}{Remark}

\newtheorem{assumption}{Assumption}[section]
\usepackage{graphicx}
\usepackage{subfig}
\usepackage{float}
\usepackage[]{algorithm2e}
\usepackage{listings}
\usepackage[toc,page]{appendix}
\usepackage[table,xcdraw]{xcolor}
\newcommand{\E}{\mathbb{E}}
\newcommand{\R}{\mathbb{R}}

\newcommand{\PS}{\mathcal{P}}

\newcommand{\x}{\mathbf{x}}

\newcommand{\w}{\mathbf{w}}
\renewcommand{\a}{\mathbf{a}}
\renewcommand{\H}{\mathcal{H}}
\newcommand{\lb}{\langle}
\newcommand{\rb}{\rangle}

\newcommand{\sprt}{\text{sprt}}
\usepackage{hyperref}
\allowdisplaybreaks
\usepackage{authblk}
\usepackage{xcolor}

\title{Generalization and Memorization: The Bias Potential Model}
\author[1]{Hongkang Yang\thanks{hongkang@princeton.edu}}
\author[1,2]{Weinan E\thanks{weinan@math.princeton.edu}}
\affil[1]{Program in Applied and Computational Mathematics, Princeton University}
\affil[2]{Department of Mathematics, Princeton University}

\begin{document}

\maketitle

\begin{abstract}
Models for learning probability distributions such as generative models and density estimators behave quite differently from models for learning functions.
One example is found in the memorization phenomenon, namely the ultimate convergence to the empirical distribution, that
occurs in generative adversarial networks (GANs). 
For this reason, the issue of generalization is more subtle than that for supervised learning.
For the bias potential model, we show that dimension-independent generalization accuracy is achievable if early stopping is adopted,
despite that in the long term, the model either memorizes the samples or diverges.
\end{abstract}



\section{Introduction}

Distribution learning models such as GANs have achieved immense popularity from their empirical success in learning complex high-dimensional probability distributions, and they have found diverse applications such as generating images \cite{brock2018BigGAN} and paintings \cite{elgammal2017can}, writing articles \cite{brown2020language}, composing music \cite{oord2016wavenet}, editing photos \cite{bau2019edit}, designing new drugs \cite{prykhodko2019molecular} and new materials \cite{mao2020material}, generating spin configurations \cite{zhang2018monge} and modeling quantum gases \cite{casert2020optical}, to name a few.

As a mathematical problem, distribution learning is much less understood. Arguably, the most fundamental question is the generalization ability of these models.  One puzzling issue is the following.
\begin{enumerate}
\item[0.] Generalization vs. memorization:

Let $Q_*$ be the target distribution, and $Q_*^{(n)}$ be the empirical distribution associated with $n$ sample points.
Let  $Q(f)$ be the probability distribution generated by some machine learning model parametrized by the function  $f$
in some hypothesis space $ \mathcal{F}$. It has been argued, for example in the case of GAN, that as training proceeds, one has
\cite{goodfellow2014generative}
\begin{equation}
\label{distribution learning limit}
\lim_{t \rightarrow \infty} Q(f(t)) = Q_*^{(n)}
\end{equation}
where $f(t)$ is the parameter we obtain at training step $t$.
We refer to (\ref{distribution learning limit}) as the ``memorization phenomenon''.
When it happens, the  model learned does not give us anything other than the samples we already have.

This is in sharp contrast to supervised learning, where models are typically trained till interpolation and can generalize well to unseen data both in practice \cite{zhang2016understanding} and in theory \cite{e2019min}.

\end{enumerate}
Despite this, these distribution-learning models perform surprisingly well in practice, being able to come close to the unseen target $Q_*$ 
and allowing us to generate new samples. This counterintuitive result calls for a closer examination of their training dynamics, 
beyond the statement (\ref{distribution learning limit}).

There are many other mysteries for distribution learning, and we list a few below.
\begin{enumerate}
\item Curse of dimensionality:

The superb performance of these models (e.g. on generating high-resolution, lifelike and diverse images \cite{brock2018BigGAN,donahue2019BigBiGAN,vahdat2020NVAE}) indicates that they can approximate the target $Q_*$ with satisfactorily small error. Yet, in theory, this should not be possible, because to estimate a general distribution in $\R^d$ with error $\leq \epsilon$, we need $n=\epsilon^{-\Omega(d)}$ amount of samples (discussed below), which becomes astronomical for real-world tasks. For instance, the BigGAN \cite{brock2018BigGAN} was trained on the ILSVRC dataset \cite{russakovsky2015imagenet} with $\leq 10^7$ images of resolution $512\times 512$, but the theoretical sample size should be like $\gg 10^{512\times512}$.

Of course, for restricted distribution families like the Gaussians, the sample complexity is only $n=\textsf{poly}(d)$. Yet, one is really interested in complex distributions such as the distribution of facial images that \textit{a priori} do not belong to any known family, so these tasks require the models to possess not only a dimension-independent sample complexity but also the universal approximation property.

\item The fragility of the training process:

It is well-known that distribution-learning models like GANs and VAE (variational autoencoder) are difficult to train.
They are  especially vulnerable to issues like mode collapse \cite{che2016mode,salimans2016improved}, instability and oscillation \cite{radford2015unsupervised}, and vanishing gradient \cite{arjovsky2017principled}.
The current treatment is to find by trial-and-error a delicate combination of the right architectures and hyper-parameters \cite{radford2015unsupervised}.
The need to understand these issues calls for a mathematical treatment.

\end{enumerate}

\vspace{1em}
This paper  offers a partial answer to these questions. We focus on the bias potential model, an expressive distribution-learning model that is relatively transparent, and uncover the mechanisms for its  generalization ability.

Specifically, we establish a dimension-independent \textit{a priori} generalization error estimate with early-stopping.
With appropriate function spaces $f\in\mathcal{F}$, the training process consists of two regimes:
\begin{itemize}
\item First, by implicit regularization, the training trajectory $Q(f(t))$ comes very close to the unseen target $Q_*$, and this is when early-stopping should be performed.
\item Afterwards, $Q(f(t))$ either converges to the sample distribution $Q_*^{(n)}$ or it diverges. 
\end{itemize}



This paper is structured as follows.
In Section \ref{sec. potential model}, we introduce the bias potential model and pose it as a continuous calculus of variations problem.
Section \ref{sec. training dynamics} analyzes the training behavior of this model and presents this paper's main results on generalization error and memorization.
Section \ref{sec. experiments} presents some numerical examples.
Section \ref{sec. proofs} contains all the proofs.
Section \ref{sec. discussions} concludes this paper with remarks on future directions.

Notation: denote vectors by bold letters $\x$. Let $C(K)$ be the space of continuous functions over some subset $K \subseteq \R^d$ equipped with supremum norm. Let $\mathcal{P}(K), \mathcal{P}_{ac}(K), \mathcal{P}_2(K)$ be the space of probability measures over $K$, the subset of absolutely continuous measures, and the subset of measures with finite second moments. 
Denote the support of a distribution $Q \in \PS(K)$ by $\sprt Q$.
Let $W_2$ be the Wasserstein metric over $\PS_2(K)$.

\subsection{Related works}
\begin{itemize}
\item Generalization ability: Among distribution-learning models, GANs have attracted the most attention and their generalization ability has been discussed in \cite{arora2017generalization,zhang2017discrimination,bai2019approximability,gulrajani2020towards} from the perspective of the neural network-based distances.
For trained models, dimension-independent generalization error estimates have been obtained only for certain restricted models, such as GANs whose generators are linear maps or one-layer networks \cite{wu2019onelayer,lei2020sgd,feizi2020LQG}.

\item Curse of dimensionality (CoD):
If the sampling error is measured by the Wasserstein metric $W_2$, then for any absolutely continuous $Q_*$ and any $\delta > 0$, it always holds that \cite{weed2017sharp}
\begin{equation*}
W_2(Q_*^{(n)},Q_*) \gtrsim n^{-\frac{1}{d-\delta}}
\end{equation*}
To achieve an error of $\epsilon$, the required sample size is $n=\epsilon^{-\Omega(d)}$.  

If sampling error is measured by KL divergence, then $KL(Q_*\|Q^{(n)}_*) = \infty$ since $Q_*^{(n)}$ is singular. If kernel smoothing is applied to $Q^{(n)}_*$, it is known that the error scales like $O(n^{-\frac{4}{d+4}})$ \cite{wand1994kernel} (technically the 
norm used in \cite{wand1994kernel} is the $L^2$ difference between densities, but one should expect that CoD would likewise be
present for KL divergence.)

\item Continuous perspective: \cite{e2019machine,e2020NNML} provide a framework to study supervised learning as continuous calculus of variations problems, with emphasis on the role of the function representation, e.g. continuous neural networks \cite{e2019barron}.
In particular, the function representation largely determines the trainability \cite{chizat2018global,rotskoff2019trainability} and generalization ability \cite{e2018priori,e2019residual,e2019min} of a supervised-learning model.
This framework can be applied to studying distribution learning in general, and we use it to analyze the bias potential model.


\item Exponential family:

The density function of the bias-potential model is an instance of the exponential families.
These distributions have long been applied to density estimation \cite{barron1991approximation,canu2006kernel} with theoretical guarantees \cite{yuan2012estimating,sriperumbudur2017density}.
Yet, existing theories focus only on black-box estimators, instead of the training process.
It has also been popular to adopt a mixture of exponential distributions \cite{kiefer1956infinite,jewell1982mixture,redner1984mixture}, but it will not be covered in this paper.

\end{itemize}

\section{Bias Potential Model}
\label{sec. potential model}

This section introduces the bias potential model, a simple distribution-learning model proposed by \cite{valsson2014potential,bonati2019enhanced} and also known as ``variationally enhanced sampling".

To pose a supervised learning model as a calculus of variations problem, one needs to consider four factors: function representation, training objective, training rule, and discretization \cite{e2019machine}.
For distribution learning, there is the additional factor of distribution representation, namely how probability distributions are represented through functions.  These are general issues for any distribution learning model.
For future reference, we go through these components in some detail.

\subsection{Distribution representation}

The bias potential model adopts the following representation:
\begin{equation}
\label{potential representation}
Q = \frac{1}{Z} e^{-V} P, \quad Z=\E_{P}[e^{-V}]
\end{equation}
where $V$ is some potential function and $P$ is some base distribution. This representation commonly appears in statistical mechanics as the Boltzmann distribution.
It is suitable for density estimation, and can also be applied to generative modeling via sampling techniques like MCMC, Langevin diffusion \cite{roberts1996Langevin}, hit-and-run \cite{lovasz2007hit}, etc.

Typically the partition function $Z$ can be ignored, since it is not involved in the training objectives or most of the sampling algorithms.

\subsection{Training objective}
\label{sec. training objective potential}

Since the representation (\ref{potential representation}) is defined by a density function, it is natural to define a density-based training objective. Given a target distribution $Q_*$, one convenient choice is the backward KL divergence
\begin{align*}
KL(Q_*\|Q) &= \E_{Q_*}[\log Q_*-\log P] + \E_{Q_*}[V] + \log\E_{P}[e^{-V}]
\end{align*}
An alternative way introduced in \cite{valsson2014potential} is to define the ``biased distribution"
\begin{equation*}
P_* = \frac{e^{V}Q_*}{\E_{Q_*}[e^V]}
\end{equation*}
so that $Q=Q_*$ iff $P=P_*$. Then, we can define an objective by the forward KL
\begin{align*}
KL(P\|P_*) &= \E_{P}[\log P-\log Q_*] - \E_{P}[V] + \log \E_{Q_*}[e^{V}]
\end{align*}

Removing constant terms, we obtain the following objectives
\begin{align}
\label{potential model objectives}
\begin{split}
L^-(V) &:= \E_{Q_*}[V] + \log\E_{P}[e^{-V}]\\
L^+(V) &:= -\E_{P}[V] + \log \E_{Q_*}[e^{V}]
\end{split} 
\end{align}
Both objectives are convex in $V$ (Lemma \ref{lemma. convexity of potential objectives}). Suppose $Q_*$ can be written as (\ref{potential representation}) with potential $V_*$, then $Q=Q_*$ iff $V=V_*+c$ for some constant $c$ iff $L^+(V)$ or $L^-(V) = 0$, so we have a unique global minimizer up to constants. Otherwise, if $Q_*$ does not have the form (\ref{potential representation}), then the minimizer does not exist.

In practice, when $Q_*$ is available only through its sample $Q_*^{(n)}$, we simply substitute all the expectation terms $\E_{Q_*}$ in (\ref{potential model objectives}) by $\E_{Q_*^{(n)}}$.

\subsection{Function representation}
\label{sec. function representation potential}

A good function representation (or function space $\mathcal{F}$) should have two conflicting properties:
\begin{enumerate}
\item $\mathcal{F}$ is expressive so that distributions generated by $f\in\mathcal{F}$ satisfy universal approximation property.
\item $\mathcal{F}$ has small complexity so that the  generalization gap is small.
\end{enumerate}
One approach is to adopt an integral transform-based representation \cite{e2019machine},
\begin{equation*}
V(\x) = \E_{\rho(\theta)} \big[ \phi(\x, \theta) \big]
\end{equation*}
for some feature function $\phi(\cdot;\theta)$ and parameter distribution $\rho$. Then, $V$ can be approximated with Monte-Carlo rate by
\begin{equation}
\label{Monte-Carlo approximation}
V_m(\x) = \frac{1}{m} \sum_{j=1}^m \phi(\x;\theta_i)
\end{equation}
where  $\{\theta_j\} $ are i.i.d. samples from $\rho(\theta)$.

Let us consider function representations built from neural networks:
\begin{itemize}
\item 2-layer neural networks: Define the continuous 2-layer network by
\begin{equation}
\label{continuous 2-NN}
V(\x) = \E_{\rho(a,\w,b)} \big[ a~\sigma(\w\cdot \x + b) \big]
\end{equation}
with an activation function $\sigma:\R\to\R$ and weights $\w \in \R^d$ and $a,b \in \R$. The natural functional norm is the Barron norm \cite{e2019barron}:
\begin{align}
\label{Barron norm}
\begin{split}
\|V\|_{\mathcal{B}} &:= \inf_{\rho} \|\rho\|_P, \quad
\|\rho\|_P^2 := \E_{\rho(a,\w,b)}\big[ a^2 (\|\w\|^2+b^2) \big]
\end{split}
\end{align}
where $\rho$ ranges over all parameter distributions that satisfy (\ref{continuous 2-NN}) and $\|\rho\|_P$ is known as the path norm.

\item Random feature model: Rewrite (\ref{continuous 2-NN}) as
\begin{equation}
\label{continuous RFM}
V(\x) = \E_{\rho_0(\w,b)} \big[ a(\w,b)~\sigma(\w\cdot \x+b) \big]
\end{equation}
with fixed parameter distribution $\rho_0(\w,b)$ and
\begin{equation*}
a(\w,b) := \frac{d\int a ~d\rho(a,\w,b)}{d\int \rho(a,\w,b) da}
\end{equation*}
The natural functional norm is the RHKS (reproducing kernel Hilbert space) norm \cite{e2019barron,rahimi2008uniform}:
\begin{equation}
\label{RKHS norm}
\|V\|_{\mathcal{H}}^2 := \E_{\rho_0} \big[a(\w,b)^2\big] = \|a\|^2_{L^2(\rho_0)}
\end{equation}
It corresponds to the RKHS $\mathcal{H}$ induced by the kernel
\begin{equation*}
k(\x,\x') = \E_{\rho_0(\w,b)}[\sigma(\w\cdot\x+b)\sigma(\w\cdot\x'+b)]
\end{equation*}
\end{itemize}

It is straightforward to establish the universal approximation theorem for these two representations and we provide such results below:
Denote by $\mathcal{P}_{ac}(K) \cap C(K)$ the distributions over $K$ with continuous density functions, and by $\|\cdot\|_{TV}$ the total variation distance, which is equivalent to the $L_1$ norm when restricted to $\mathcal{P}_{ac}(\R^d)$.

\begin{proposition}[Universal approximation]
\label{prop. universal approximation potential}
Let $K \subseteq \R^d$ be any compact set with positive Lebesgue measure, let $P$ be the uniform distribution over $K$, and let $\mathcal{V}$ be any class of functions that is dense in $C(K)$. Then, the class of probability distributions (\ref{potential representation}) generated by $V\in\mathcal{V}$ and $P$ are
\begin{itemize}
\item dense in $\mathcal{P}(K)$ under the Wasserstein metric $W_p$ ($1\leq p < \infty$),
\item dense in $\mathcal{P}_{ac}(K)$ under the total variation norm $\|\cdot\|_{TV}$,
\item dense in $\mathcal{P}_{ac}(K) \cap C(K)$ under KL divergence.
\end{itemize}
Given assumption \ref{assumption. universal approximation potential}, this result applies if $\mathcal{V}$ is the Barron space $\{\|V\|_{\mathcal{B}} < \infty\}$ or RKHS space $\{\|V\|_{\mathcal{H}} < \infty\}$.
\end{proposition}


The Monte-Carlo approximation (\ref{Monte-Carlo approximation}) suggests that these continuous models can be approximated efficiently by finite neural networks. Specifically, we can establish the following \textit{a priori} error estimates:

\begin{proposition}[Efficient approximation]
\label{prop. approximation error potential}
Suppose that the base distribution $P$ is compactly-supported in a ball $B_R(0)$, and the activation function $\sigma$ is Lipschitz with $\sigma(0)=0$.
Given $\|V\|_{\mathcal{B}}<\infty$, for every $m\in\mathbb{N}$, there exists a finite 2-layer network $V_m$ with $m$ neurons that 
satisfies:
\begin{align*}
KL(Q\|Q_m) &\leq \frac{\|V\|_{\mathcal{B}}}{\sqrt{m}} \cdot 2\sqrt{3} \|\sigma\|_{Lip} \sqrt{R^2+1}\\
\|V_m\|_{\mathcal{B}} &\leq \sqrt{2} \|V\|_{\mathcal{B}}
\end{align*}
where $Q_m$ is the distribution generated by $V_m$.
Similarly, assume that the fixed parameter distribution $\rho_0$ in (\ref{continuous RFM}) is compactly-supported in a ball $B_r(0)$, then given $\|V\|_{\mathcal{H}}<\infty$, for every $m$, there exists $V_m$ such that
\begin{align*}
KL(Q\|Q_m) &\leq \frac{\|V\|_{\mathcal{H}}}{\sqrt{m}} \cdot 2\sqrt{3} \|\sigma\|_{Lip} \sqrt{R^2+1} ~r\\
\|V_m\|_{\mathcal{H}} &\leq \sqrt{2} \|V\|_{\mathcal{H}}
\end{align*}
\end{proposition}


\subsection{Training rule}
\label{sec. training rule potential}

We consider the simplest training rule, the gradient flow. 

For continuous function representations, there are generally two kinds of flows:
\begin{itemize}
\item Non-conservative gradient flow

For the random feature model (\ref{continuous RFM}), we can train the function $a(\w,b)$ using its variational gradient
\begin{equation*}
\partial_t a(\w,b) = -\frac{\delta L}{\delta a}(\w,b)
\end{equation*}
Specifically, for the training objectives $L^{\pm}(V)$ in (\ref{potential model objectives}), the corresponding flows are defined by
\begin{align*}
\frac{d}{dt} a &= - \frac{\delta L^+}{\delta a}
= -\E_{P_*-P}[\sigma(\w\cdot\x+b)]\\
\frac{d}{dt} a &= - \frac{\delta L^-}{\delta a}
= -\E_{Q_*-Q}[\sigma(\w\cdot\x+b)]
\end{align*}

\item Conservative gradient flow:

For the 2-layer neural network (\ref{continuous 2-NN}), we train the parameter distribution $\rho(a,\w,b)$. Its gradient flow is constrained by the conservation of local mass and obeys the continuity equation (in the weak sense):
\begin{equation}
\label{continuity equation}
\partial_t \rho - \nabla \cdot \big(\rho \nabla \frac{\delta L}{\delta \rho}\big) = 0
\end{equation}
With the objectives (\ref{potential model objectives}), the gradient fields are given by
\begin{align}
\label{potential 2-NN gradient flow}
\begin{split}
\nabla \frac{\delta L^+}{\delta \rho} &= \nabla_{(a,\w,b)} \E_{P_*-P}\big[ a~\sigma(\w\cdot\x+b) \big] = \E_{P_*-P} \begin{bmatrix}
\sigma(\w\cdot\x+b)\\ a~\sigma'(\w\cdot\x+b)~\x \\ a~\sigma'(\w\cdot\x+b)
\end{bmatrix}\\
\nabla \frac{\delta L^-}{\delta \rho} &= \nabla_{(a,\w,b)} \E_{Q_*-Q}\big[ a~\sigma(\w\cdot\x+b) \big] = \E_{Q_*-Q} \begin{bmatrix}
\sigma(\w\cdot\x+b)\\ a~\sigma'(\w\cdot\x+b)~\x \\ a~\sigma'(\w\cdot\x+b)
\end{bmatrix}
\end{split}
\end{align}

\end{itemize}

\subsection{Discretization}
\label{sec. discretization potential}

So far we have only discussed the continuous formulation of distribution learning models. In practice, we implement these continuous models using discretized versions, with the hope that the discretized models inherit these properties up to a controllable discretization error.

Let us focus on the discretization in the parameter space, and in particular, the most popular ``particle discretization", since this is the analog of Monte-Carlo
for dynamic problems.
Consider the parameter distribution $\rho(a,\w,b)$ of the 2-layer net (\ref{continuous 2-NN}) and its approximation
by the empirical distribution
\begin{equation*}
\rho^{(m)} = \frac{1}{m} \sum_{j=1}^m \delta_{(a_j,\w_j,b_j)}
\end{equation*}
where the particles $\{(a_j,\w_j,b_j) \}$ are i.i.d. samples of $\rho$. The potential  function represented by this
empirical distribution is given by:
\begin{equation*}
V_m(\x) = \E_{\rho^{(m)}}\big[a~\sigma(\w\cdot\x+b)\big] = \frac{1}{m} \sum_j a_j~ \sigma(\w_j \cdot \x + b_j)
\end{equation*}
Suppose we train $\rho^{(m)}$ by conservative gradient flow (\ref{continuity equation},\ref{potential 2-NN gradient flow}) with the objective $L^-$. The continuity equation (\ref{continuity equation}) implies that, for any smooth test function $f(a,\w,b)$, we have
\begin{align*}
\frac{d}{dt} \int f ~d\rho^{(m)} = -\int \nabla f \cdot \nabla \frac{\delta L}{\delta \rho} d\rho^{(m)} = - \frac{1}{m} \sum_{j=1}^m \nabla f(a_j,\w_j,b_j)^T \cdot \nabla \frac{\delta L}{\delta \rho}(a_j,\w_j,b_j)
\end{align*}
Meanwhile, we also have
\begin{align*}
\frac{d}{dt} \int f ~d\rho^{(m)} = \frac{1}{m} \sum_{j=1}^m \nabla f(a_j,\w_j,b_j)^T \cdot \frac{d}{dt} \begin{bmatrix} a_j\\ \w_j\\ b_j \end{bmatrix}
\end{align*}
Thus we have recovered the  gradient flow for finite \textit{scaled} 2-layer networks:
\begin{align*}
\frac{d}{dt} \begin{bmatrix} \a_j\\ \w_j\\ b_j \end{bmatrix}
&= -\nabla \frac{\delta L^-(V_m)}{\delta \rho}(a_j,\w_j,b_j) = -m\cdot\frac{\partial L^-(V_m)}{\partial (a_j,\w_j,b_j)}
= -\E_{Q_*-Q} \begin{bmatrix}
\sigma(\w_j\cdot\x+b_j)\\ a_j~\sigma'(\w_j\cdot\x+b_j)~\x \\ a_j~\sigma'(\w_j\cdot\x+b_j)
\end{bmatrix}
\end{align*}
This example shows that the particle discretization of continuous 2-layer networks (\ref{continuous 2-NN}) leads to the same result as the mean-field modeling of 2-layer nets \cite{mei2018mean,rotskoff2019trainability}.

\section{Training Dynamics}
\label{sec. training dynamics}

This section studies the training behavior of the bias potential model and presents the main result of this paper, on the relation between generalization and memorization: When trained on a finite sample set,
\begin{itemize}
\item With early stopping, the model reaches dimension-independent generalization error rate.
\item As $t\to\infty$, the model necessarily memorizes the samples unless it diverges.
\end{itemize}

\subsection{Trainability}

We begin with the training dynamics on the population loss.
First, we consider the random feature model (\ref{continuous RFM}) and establish global convergence:
\begin{proposition}[Trainability]
\label{prop. bias potential RFM}
Suppose that the target distribution $Q_*$ is generated by a potential $V_*$ ($\|V_*\|_{\mathcal{H}} < \infty$).
Suppose that our distribution $Q_t$ is generated by potential $V_t$ with parameter function $a_t$ trained by gradient flow on either of the objectives (\ref{potential model objectives}).
Then, 
\begin{equation*}
L^{\pm}(V_t)-L^{\pm}(V_*) \leq \frac{||V_*-V_0||^2_{\mathcal{H}}}{2t}
\end{equation*}
\end{proposition}


Next, for 2-layer neural networks, we show that whenever the conservative gradient flow converges, it must converge to the global minimizer.
In particular, it will not be trapped at bad local minima and thus avoids mode collapse.
This result is analogous to the global optimality guarantees for supervised learning and regression problems \cite{chizat2018global,rotskoff2019trainability}.
\begin{proposition}
\label{prop. bias potential 2-NN}
Assume that the distribution $Q_t$ is generated by potential $V_t$, a 2-layer network with parameter distribution $\rho_t$ trained by gradient flow on either of the objectives (\ref{potential model objectives}).
Assume that the assumption \ref{assumption. potential activation} holds.  If the flow $\rho_t$ converges in $W_1$ metric (or any $W_p$, $1\leq p \leq \infty$) to some $\rho_{\infty}$ as $t\to\infty$, then $\rho_{\infty}$ is a global minimizer of $L^{\pm}$: Let $V_{\infty}$ be the corresponding 2-layer network, then
\begin{equation*}
Q_* = Q_{\infty} = \frac{e^{-V_{\infty}}P}{\E_{P}[e^{-V_{\infty}}]}
\end{equation*}
\end{proposition}


\subsection{Generalization ability}
\label{sec. generalization potential}

Now we consider the most important issue for the model,  the generalization error, and prove that a dimension-independent \textit{a priori} error rate is achievable within a convenient early-stopping time interval.

We study the training dynamics on the empirical loss.
For convenience, we make the following assumptions:
\begin{itemize}
\item  Let the base distribution $P$ in (\ref{potential representation}) be supported on $[-1,1]^d$ (the $l^{\infty}$ ball).
Without loss of generality, we use the $l^{\infty}$ norm on $[-1, 1]^d$.

\item Let the objective $L$ be $L^-(V)$ from (\ref{potential model objectives}) (The analysis of $L^+$ would be more involved). Recall that if the target $Q_*$ is generated by a potential $V_*$, then
\begin{equation*}
L(V)-L(V_*) = KL(Q_*\|Q)
\end{equation*}
Denote by $L^{(n)}$ the empirical loss that corresponds to $Q_*^{(n)}$:
\begin{equation*}
L^{(n)}(V) = \E_{Q^{(n)}_*}[V] + \log\E_{P}[e^{-V}] = L(V) + \E_{Q^{(n)}_*-Q_*}[V]
\end{equation*}

\item Model $V$ by the random feature model (\ref{continuous RFM}) with RKHS norm $\|V\|_{\mathcal{H}}=\|a\|_{L^2(\rho_0)}$ from (\ref{RKHS norm}). Assume that the activation function $\sigma$ is ReLU, and that the fixed parameter distribution $\rho_0$ is supported inside the $l^1$ ball, that is, $\|\w\|_1+|b|\leq 1$ for $\rho_0$ almost all $(\w,b)$. Denote $(\x,1)$ by $\tilde{\x}$ and $(\w,b)$ by $\w$, so the activation can be written as $\sigma(\w\cdot\tilde{\x})$.

\begin{remark}[Universal approximation]
\label{remark. universal approximation}
If we further assume that $\rho_0$ covers all directions (e.g. $\rho_0$ is uniform over the $l^1$ sphere $\{\|\w\|_1+|b|=1\}$) and $P$ is uniform over some $K \subseteq [-1,1]^d$, then Proposition \ref{prop. universal approximation potential} implies that this model enjoys universal approximation over distributions on $K$.
\end{remark}

\item Training rule: We train $a$ by gradient flow (Section \ref{sec. training rule potential}). Let $a_t,V_t,Q_t$ and $a_t^{(n)},V_t^{(n)}, Q_t^{(n)}$ be the training trajectories under $L$ and $L^{(n)}$. Assume the same initialization $a_0=a_0^{(n)}$.
\end{itemize}

\begin{theorem}[Generalization ability]
\label{thm. generalization for potential model}
Suppose $Q_*$ is generated by a potential function $V_*$ ($\|V_*\|_{\mathcal{H}}<\infty$). For any $\delta \in (0,1)$, with probability $1-\delta$ over the sampling of $Q_*^{(n)}$, the testing error of $Q_t^{(n)}$ is bounded by
\begin{equation*}
KL\big(Q_*\|Q_t^{(n)}\big) \leq \frac{\|V_*-V_0\|^2_{\mathcal{H}}}{2t} + 2\Big(4\frac{\sqrt{2\log 2d}}{\sqrt{n}} + \frac{\sqrt{2\log(2/\delta)}}{\sqrt{n}} \Big) t
\end{equation*}
\end{theorem}

\begin{corollary}
\label{cor. generalization error potential}
Given the condition of Theorem \ref{thm. generalization for potential model}, if we choose an early-stopping time $T$ such that
\begin{equation*}
T = \Theta\Big(\|V_*-V_0\|_{\mathcal{H}}\big(\frac{n}{\log d}\big)^{1/4}\Big)
\end{equation*}
then the testing error obeys
\begin{equation*}
KL\big(Q_*\|Q_T^{(n)}\big) \lesssim \|V_*-V_0\|_{\mathcal{H}} \Big(\frac{\log d}{n}\Big)^{1/4}
\end{equation*}
\end{corollary}

\vspace{0.5em}
This rate is significant in that it is dimension-independent up to a negligible $(\log d)^{1/4}$ term. Although the upper bound $n^{-1/4}$ is slower than the desirable Monte-Carlo rate of $n^{-1/2}$, it is much better than the rate $n^{-1/d}$ and we believe there is room for improvement. In addition, the early-stopping time interval is reachable within a time that is dimension-independent and the width of this interval is at least on the order of $n^{1/4}$.

This result is enabled by the function representation of the model, specifically:
\begin{enumerate}
\item Learnability: If the target $Q_*$ lives in the right space for our function representation, then the optimization rate (for the population loss $L(V_t)-L(V_*)$) is fast and dimension-independent. In this case, the right space consists of distributions generated by random feature models, and the $O(1/t)$ rate is provided by Proposition \ref{prop. bias potential RFM}.

\item Insensitivity to high dimensional structures: The function representations have small Rademacher complexity, so they are insensitive to the empirical error $Q_*-Q_*^{(n)}$ and the resulting deviation of the training trajectory $Q^{(n)}_t-Q_t$ scales as $O(n^{-1/2})$ instead of $O(n^{-1/d})$. This result is provided by Lemmas \ref{lemma. gradient Monte Carlo rate} and \ref{lemma. perturbed convex gradient flow} below. 
\end{enumerate}

\begin{lemma}
\label{lemma. gradient Monte Carlo rate}
For any distribution $Q_*$ supported on $[-1,1]^d$ and any $\delta \in (0,1)$, with probability $1-\delta$ over the i.i.d. sampling of the empirical distribution $Q_*^{(n)}$, we have
\begin{align*}
\sup_{\|\w\|_1 \leq 1} \E_{Q_*-Q_*^{(n)}}[\sigma(\w\cdot\tilde{\x})] \leq 4\sqrt{\frac{2\log 2d}{n}} + \sqrt{\frac{2\log (2/\delta)}{n}}
\end{align*}
\end{lemma}

\begin{lemma}
\label{lemma. perturbed convex gradient flow}
Let $L$ be a convex Fr\'{e}chet-differentiable function over a Hilbert space $H$ with Lipschitz constant $l$. Let $h$ be a Fr\'{e}chet-differentiable function with Lipschitz constant $\epsilon$.
Define two gradient flow trajectories $x_t,y_t$:
\begin{align*}
x_0 = y_0, ~\frac{d}{dt}x_t = -\nabla L(x_t), ~\frac{d}{dt}y_t = - \nabla \widetilde{L}(y_t)
\end{align*}
where $\widetilde{L} = L + h$ represents a perturbed function. Then,
\begin{equation*}
L(y_t) - L(x_t) \leq l\epsilon t
\end{equation*}
for all time $t \geq 0$.
\end{lemma}

Numerical examples for the training process and generalization error are provided in Section \ref{sec. experiments}.

\subsection{Memorization}

Despite that the model enjoys good generalization accuracy with early stopping, we show that in the long term the solution $Q_t^{(n)}$ necessarily deteriorates.

\begin{proposition}[Memorization]
\label{prop. memorization}
Under the condition of Theorem \ref{thm. generalization for potential model} and Remark \ref{remark. universal approximation},
\begin{enumerate}
\item If the trajectory $Q_t^{(n)}$ has only one weak limit, then $Q_t^{(n)}$ converges weakly to the empirical distribution $Q_*^{(n)}$.
\item The true target distribution $Q_*$ can never be a limit point of $Q_t^{(n)}$. The generalization error and the potential function's norm both diverge
\begin{equation*}
\lim_{t\to\infty} KL(Q_*\|Q_t^{(n)}) = \lim_{t\to\infty} \|V_t^{(n)}\|_{\mathcal{H}} = \infty
\end{equation*}
\end{enumerate}
\end{proposition}

\noindent
Hence, the model either memorizes the samples or diverges (coming to more than one limit, which are all degenerate), even though it may not manifest within realistic training time.

The proof is based on the following observation.
\begin{lemma}
\label{lemma. universal convergence}
Let $K\subseteq \R^d$ be a compact set with positive Lebesgue measure, let the base distribution $P$ be uniform over $K$, and let $k$ be a continuous and integrally strictly positive definite kernel on $K$.
Given any target distribution $Q' \in \PS(K)$ and any initialization $V_0 \in C(K)$, train the potential $V_t$ by
\begin{equation*}
\frac{d}{dt} V_t(\x) = \E_{(Q_t-Q')(\x')}[k(\x,\x')]
\end{equation*}
If $Q_t$ has only one weak limit, then $Q_t$ converges weakly to $Q'$.
Else, none of the limit points cover the support of $Q'$.
\end{lemma}

A numerical demonstration of memorization is provided in Section \ref{sec. experiments}.

\subsection{Regularization}
\label{sec. regularization}

Instead of early stopping, one can also consider explicit regularization: With the empirical loss $L^{(n)}$, define the problem
\begin{equation*}
\min_{\|V\|\leq R} L^{(n)}(V)
\end{equation*}
for some appropriate functional norm $\|\cdot\|$ and adjustable bound $R$.
For the special case of random feature models (\ref{RKHS norm}), this problem becomes
\begin{equation}
\label{constrained empirical loss}
\min_{\|a\|_{L^2(\rho_0)}\leq R} L^{(n)}(a)
\end{equation}
where $L^{(n)}(a)$ denotes $L^{(n)}(V)$ with potential $V$ generated by $a$.

By convexity, $L^{(n)}(a^{(n)}_t)$ can always converge to the minimum value as $t\to\infty$ if $a^{(n)}_t$ is trained by gradient flow constrained to the ball $\{\|a\|_{L^2(\rho_0)}\leq R\}$.
Denote the minimizer of (\ref{constrained empirical loss}) by $a_R^{(n)}$ (which exists by Lemma \ref{lemma. existence regularized minimizer}) and denote the corresponding distribution by $Q^{(n)}_R$.

\begin{proposition}
\label{prop. regularized model}
Given the condition of Theorem \ref{thm. generalization for potential model}, choose any $R \geq \|V_*\|_{\H}$. With probability $1-\delta$ over the sampling of $Q_*^{(n)}$, the minimizer $a^{(n)}_R$ satisfies
\begin{equation*}
KL(Q_*\|Q^{(n)}_R) \lesssim \frac{\sqrt{\log d}+\sqrt{\log 1/\delta}}{\sqrt{n}} R
\end{equation*}
\end{proposition}

This result can be straightforwardly extended to the case when $V, V_*$ are implemented as 2-layer networks or deep residual networks, equipped with the norms defined in \cite{e2019barron}. The proof only involves the Rademacher complexity, and it is known that these functions' complexity scales as $O(\frac{R}{\sqrt{n}})$ \cite{e2019barron}.

\section{Numerical Experiments}
\label{sec. experiments}

Corollary \ref{cor. generalization error potential} and Proposition \ref{prop. memorization} tell us that the training process roughly consists of two phases:
the first phase in which  a dimension-independent generalization error rate is reached, 
and a second phase  in which the model deteriorates into memorization or divergence.
We now examine how these happen in practice.

\subsection{Dimension-independent error rate}
\label{sec. empirical sample rate}

The key aspect of the generalization estimate of Corollary \ref{cor. generalization error potential} is that its sample complexity $O(n^{-\alpha})$ ($\alpha\geq 1/4$) is dimension-independent.

To verify dimension-independence, we estimate  the exponent $\alpha$ for varying dimension $d$. We adopt the set-up of Theorem \ref{thm. generalization for potential model} and train our model $Q^{(n)}_t$ by SGD on a finite sample set $Q_*^{(n)}$. Specifically, $P$ is uniform over $[-1,1]^d$, the target and trained distributions $Q_*, Q_t^{(n)}$ are generated by the potentials $V_*, V_t^{(n)}$, these potentials are random feature functions (\ref{continuous RFM}) with $\rho_0$ being uniform over the $l^1$ sphere $\{\|\w\|_1+|b|=1\}$, with parameter functions $a_*, a^{(n)}_t$ and ReLU activation. The samples $Q_*^{(n)}$ are obtained by Projected Langevin Monte Carlo \cite{bubeck2018LMC}. We approximate $\rho_0$ using $m=500$ samples (particle discretization) and set $a_* \equiv 50$. We initialize training with $a^{(n)}_0 \equiv 0$ and train $a^{(n)}_t$ by scaled gradient descent with learning rate $0.5 m$.

The generalization error is measured by $KL(Q_*\|Q_t^{(n)})$. Denote the optimal stopping time by
\begin{equation*}
T_o = \text{argmin}_{t>0} KL(Q_*\|Q_t^{(n)})
\end{equation*}
and the corresponding optimal error by $L_o$. The most difficult part of this experiment turned out to be the computation of the
KL divergence: Monte-Carlo approximation has led to excessive variance. Therefore we computed by numerical integration on a uniform grid on $[-1,1]^d$. This limits the experiments to low dimensions.

For each $d \leq 5$, we estimate $\alpha$ by linear regression between $\log n$ and $\log L_o$. The sample size $n$ ranges in $\{25,50,100,200\}$, each setting is repeated 20 times with a new sample set $Q_*^{(n)}$. Also, we solve for the dependence of $T_o$ on $n$ by linear regression between $\log n$ and $\log L_o$. Here are the results:

\begin{table}[H]
\centering
\begin{tabular}{ |c|c|c|c|c|c| } 
\hline
Dimension $d$ & 1 & 2 & 3 & 4 & 5 \\
\hline
Exponent $-\alpha$ of $L_o$ & $-0.74$ & $-0.71$ & $-0.81$ & $-0.74$ & $-0.87$ \\
Exponent of $T_o$ & 0.30 & 0.29 & 0.27 & 0.26 & 0.31 \\
\hline
\end{tabular}
\caption{Upper: empirically, the exponent $\alpha$ of the sample complexity is dimension-independent. Lower: the optimal stopping time grows with $n$.}
\end{table}

Our experiments suggest that the generalization error of the early-stopping solution scales as $n^{-0.8}$ and is dimension-independent, and the optimal early-stopping time is around $n^{0.3}$. This error is much better than the upper bound $O(n^{-1/4})$ given by Corollary \ref{cor. generalization error potential}, indicating that our analysis has much room for improvement.

Shown in Figure  \ref{fig: test error curve} is the generalization error $KL(Q_*\|Q_t^{(n)})$ during training, for dimension $d=5$.
\begin{figure}[H]
    \centering
    \includegraphics[scale=0.45]{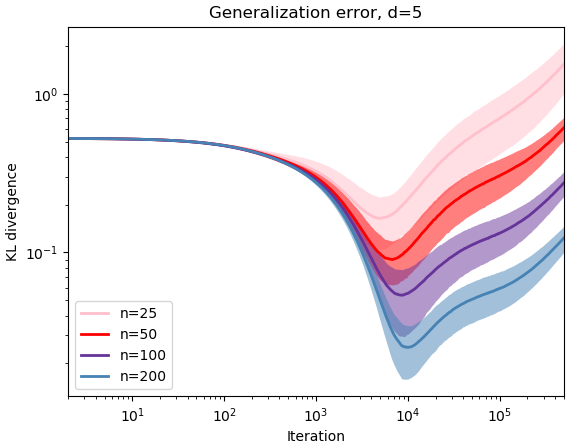}
    \caption{Generalization error curves with log axes. The solid curves are averages over 20 trials, and the shaded regions are $\pm 1$ standard deviations. The results for other $d$ are similar.}
    \label{fig: test error curve}
\end{figure}
All error curves go through a rapid descent, followed by a slower but gradual ascent due to memorization. In fact, the convergence rate prior to the optimal stopping time appears to be exponential. 
Note that if exponential convergence indeed holds, then the generalization error estimate of Corollary \ref{cor. generalization error potential} can be improved to $O(n^{-1/2} \log n)$.

\subsection{Deterioration and memorization}

Proposition \ref{prop. memorization} indicates that as $t\to\infty$ the model either memorizes the sample points or diverges. Our result shows that in practice we  obtain memorization.

We adopt the same set-up as in Section \ref{sec. empirical sample rate}. Since memorization occurs very slowly with SGD, we accelerate training using Adam. 
Figure \ref{fig: long time solution and memorization} shows the result for $d=1, n=25$. 
\begin{figure}[H]
\centering
\subfloat{\includegraphics[scale=0.3]{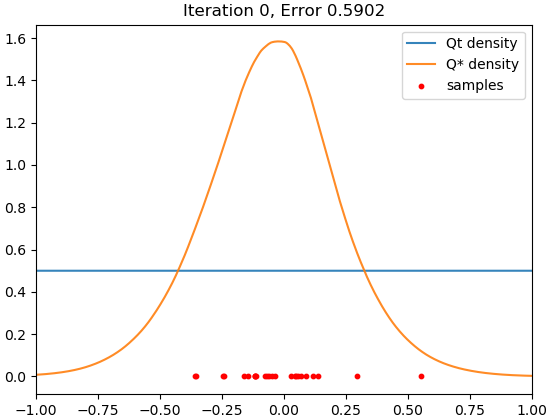}}
\subfloat{\includegraphics[scale=0.3]{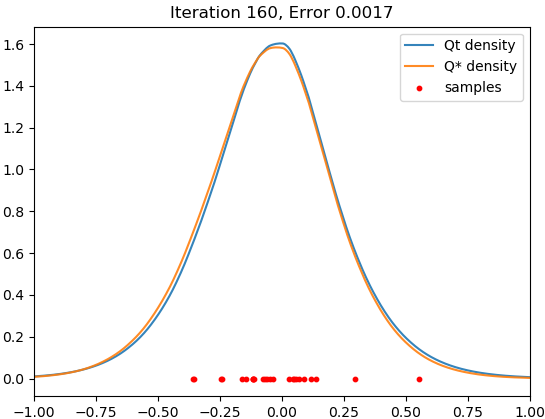}}
\subfloat{\includegraphics[scale=0.3]{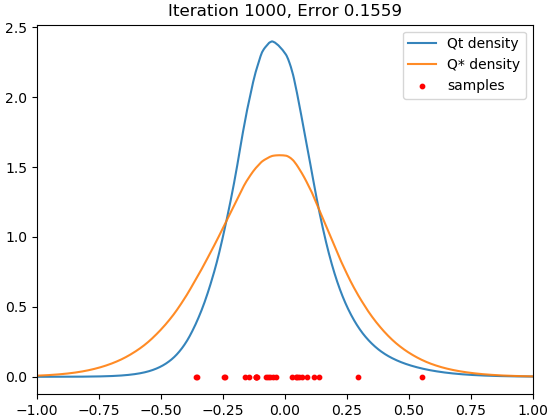}}\\
\subfloat{\includegraphics[scale=0.3]{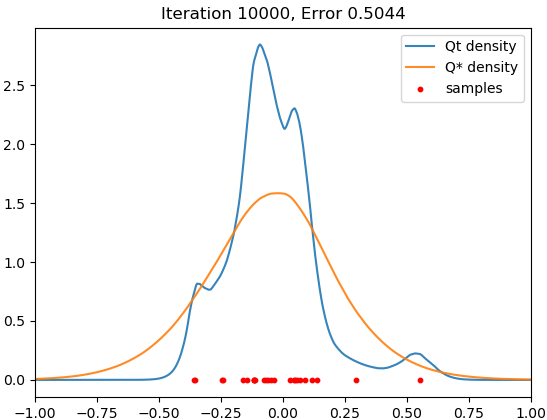}}
\subfloat{\includegraphics[scale=0.3]{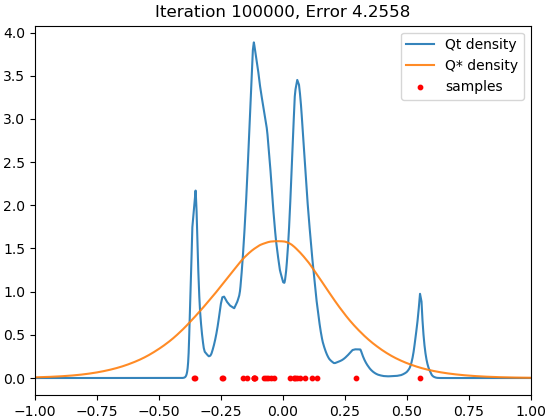}}
\subfloat{\includegraphics[scale=0.3]{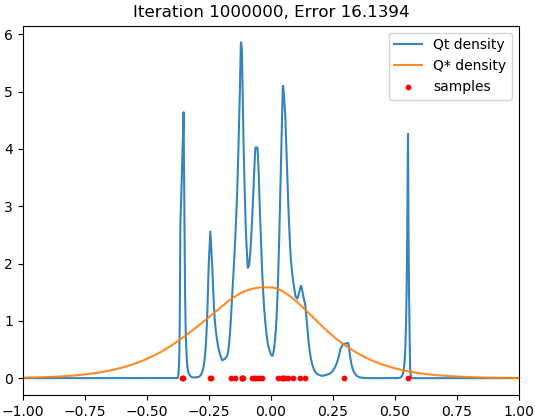}}
\caption{From top left to bottom right: Initialization, optimal stopping time at iteration $160$, long time solutions at iterations $10^3, 10^4, 10^5$ and $10^6$. The orange curve is the density of the target distribution $Q_*$, and the blue curves are $Q_t^{(n)}$. The red dots are the samples $Q_*^{(n)}$.}
\label{fig: long time solution and memorization}
\end{figure}

We see that  there is a time interval during which the trained model closely fits the target distribution, but it eventually concentrates around the samples, and this memorization process does not seem to halt within realistic training time.

Figure \ref{fig: test error and RKHS norm} suggests that memorization is correlated with the growth of the function norm of the potential.
\begin{figure}[H]
    \centering
    \includegraphics[scale=0.45]{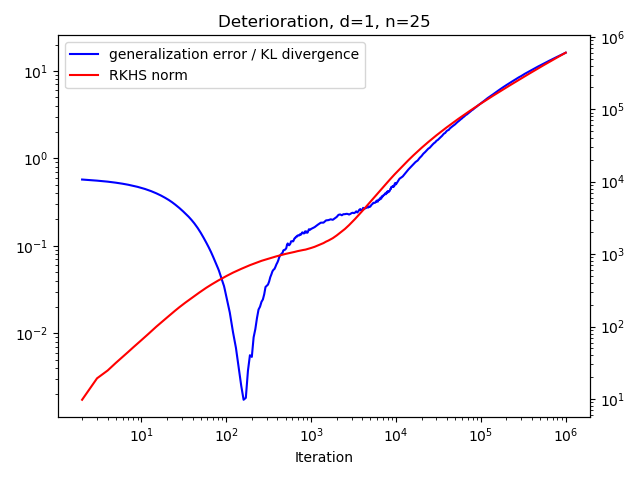}
    \caption{Left: generalization error with Adam optimizer. Right: RKHS norm $\|V^{(n)}_t\|_{\mathcal{H}}$.}
    \label{fig: test error and RKHS norm}
\end{figure}

\section{Proofs}
\label{sec. proofs}

\subsection{Proof of the Universal Approximation Property}
\label{appendix. universal approximation potential}

\begin{assumption}
\label{assumption. universal approximation potential}
For 2-layer neural networks (\ref{continuous 2-NN}),  assume that the activation function $\sigma:\R\to\R$ is continuous and is not a polynomial.

For the random feature model (\ref{continuous RFM}),  assume that the activation function is continuous, non-polynomial and grows at most linearly at infinity, $\sigma(x) = O(|x|)$. In addition, we assume that  the fixed parameter distribution $\rho_0(\w,b)$ has full support over $\R^{d+1}$. (See Theorem 1 of \cite{sun2018RFM} for more general conditions.) Alternatively, one can assume that $\sigma$ is ReLU and $\rho_0$ covers all directions, that is, for all $(\w,b) \neq \mathbf{0}$, we have $\lambda (\w,b) \in \sprt \rho_0$ for some $\lambda > 0$.

By Theorem 3.1 of \cite{pinkus1999approximation}, Theorem 1 and Proposition 1 of \cite{sun2018RFM}, the Barron space $\mathcal{B}$ and RKHS space $\mathcal{H}$ defined by (\ref{Barron norm}) and (\ref{RKHS norm}) are dense in the space of continuous functions over any compact subset of $\R^d$.
\end{assumption}

\begin{proof}[Proof of Proposition \ref{prop. universal approximation potential}]
Denote the set of distributions generated by $\mathcal{V}$ by
\begin{equation*}
\mathcal{Q} = \{ Q \in \PS(K) ~|~ Q \text{ is given by (\ref{potential representation}) with } V\in\mathcal{V} \}
\end{equation*}

First, for any $Q_* \in \mathcal{P}_{ac}(K) \cap C(K)$, assume that its density function is strictly positive $Q_*(\x) \geq \epsilon > 0$ over $K$. Then, $\log Q_* \in C(K)$. Let $\{V_m\} \subseteq \mathcal{V}$ be a sequence that approximates $\log Q_*$ in the supremum norm, and let $Q_m$ be the distributions (\ref{potential representation}) generated by $V_m$. It follows that
\begin{equation*}
\lim_{m\to\infty} KL(Q_*\|Q_m) \leq \lim_{m\to\infty} \|\log Q_* - \log Q_m \|_{C(K)} = 0
\end{equation*}

For the general case $Q_* \in \mathcal{P}_{ac}(K) \cap C(K)$, define for any $\epsilon \in (0,1)$,
\begin{equation*}
Q_*^{\epsilon} = (1-\epsilon) Q_* + \epsilon P
\end{equation*}
For each $m \in \mathbb{N}$, let $Q_m$ be a distribution generated by some $V_m \in \mathcal{V}$ such that $\|\log Q_*^{1/m} - V_m \|_{C(K)} < 1/m$. Then,
\begin{align*}
\lim_{m\to\infty} KL(Q_*\|Q_m) &= \lim_{m\to\infty} KL(Q_*\|Q_*^{1/m}) + \E_{Q_*} \log\frac{Q_*^{1/m}(\x)}{Q_m(\x)}\\
&\leq \lim_{m\to\infty} \frac{1}{m} KL(Q_*\|P) + \|\log Q_*^{1/m} - V_m \|_{C(K)} = 0
\end{align*}
where the inequality follows from the convexity of KL. Hence, the set $\mathcal{Q}$ is dense in $\mathcal{P}_{ac}(K) \cap C(K)$ under KL divergence.

Next, consider the total variation norm. Since $\mathcal{P}_{ac}(K) \cap C(K)$ is dense in $\mathcal{P}_{ac}(K)$ under $\|\cdot\|_{TV}$, and since Pinsker's inequality bounds $\|\cdot\|_{TV}$ from above by KL divergence, we conclude that $\mathcal{Q}$ is also dense in $(\mathcal{P}_{ac}(K), \|\cdot\|_{TV})$.

Now consider the $W_1$ metric. $\|\cdot\|_{TV}$ can be seen as an optimal transport cost with cost function $c(\x,\x') = \mathbf{1}_{\x\neq \x'}$, so for any $Q_1,Q_2 \in \mathcal{P}(K)$,
\begin{equation*}
W_1(Q_1,Q_2) \leq \text{diam}(K) ~\|Q_1-Q_2\|_{TV}
\end{equation*}
Since $\mathcal{P}_{ac}(K)$ is dense in $\mathcal{P}(K)$ under the $W_1$ metric, we conclude that $\mathcal{Q}$ is dense in $(\mathcal{P}(K),W_1)$.

Finally, note that for any $p \in [1,\infty)$,
\begin{equation*}
W_p \lesssim ~\text{diam}(K)^{1-1/p} ~W_1^{1/p}
\end{equation*}
So $\mathcal{Q}$ is dense in $(\mathcal{P}(K),W_p)$.
\end{proof}

\subsection{Estimating the Approximation Error}
\label{appendix. approximation error potential}

\begin{lemma}
\label{lemma. log expectation bound for potential}
For any base distribution $P$ and any potential functions $V_1,V_2$,
\begin{equation*}
\big|\log \E_P [e^{-V_1}] - \log \E_P[e^{-V_2}]\big| \leq \|V_1-V_2\|_{L^{\infty}(P)}
\end{equation*}
\end{lemma}
\begin{proof}
Denote $V_{\max},V_{\min} = \max(V_1,V_2),\min(V_1,V_2)$. Then,
\begin{align*}
&\quad \big|\log\E_P[e^{-V_1}] - \log\E_P[e^{-V_2}]\big|\\
&\leq \log\E_P[e^{-V_{\min}}] - \log\E_P[e^{-V_{\max}}]\\
&\leq \log\big(\|e^{-V_{\max}}\|_{L^1(P)} \|e^{V_{\max}-V_{\min}}\|_{L^{\infty}(P)}\big) - \log \|e^{-V_{\max}}\|_{L^1(P)}\\
&= \log \|e^{V_{\max}-V_{\min}}\|_{L^{\infty}(P)}\\
&= \|V_1-V_2\|_{L^{\infty}(P)}
\end{align*}
\end{proof}

\begin{proof}[Proof of Proposition \ref{prop. approximation error potential}]
The proof follows the standard argument of Monte-Carlo estimation (Theorem 4 of \cite{e2019barron}). First, consider the case $\|V\|_{\mathcal{B}}<\infty$. For any $\epsilon \in (0, 0.01)$, let $\rho$ be a parameter distribution of $V$ with path norm $\|\rho\|_P < (1+\epsilon)\|V\|_{\mathcal{B}}$. Define the finite neural network
\begin{equation*}
V_m(\x) = \frac{1}{m} \sum_{j=1}^m a_j \sigma(\w_j \cdot \x + b_j) =: \frac{1}{m} \sum_{j=1}^m \phi(\x;\theta_j)
\end{equation*}
where $\theta_j = (a_j,\w_j,b_j)$ are i.i.d. samples from $\rho$. Denote $\Theta = (\theta_j)_{j=1}^m$.

Let $Q_m$ be the distribution generated by $V_m$. The approximation error is given by
\begin{align*}
KL(Q\|Q_m) &= \E_Q \big[ V_m-V \big] + (\log \E_P[e^{-V_m} - \log \E_P [e^{-V}]])
\end{align*}
By Lemma \ref{lemma. log expectation bound for potential},
\begin{align*}
KL(Q\|Q_m) &\leq \|V_m-V\|_{L^{\infty}(Q)} + \|V_m-V\|_{L^{\infty}(P)}\\
&\leq 2\|V_m-V\|_{L^{\infty}(P)}
\end{align*}
Given that $\sprt P \subseteq B_R(0)$, we can bound
\begin{align*}
\E_{\Theta} \big[ \|V-V_m\|^2_{L^{\infty}(P)} \big]
&\leq \E_{\Theta} \Big[ \sup_{\|\x\|\leq R} \Big( \frac{1}{m} \sum_{j=1}^m \phi(\x;\theta_j) - \E_{\theta \sim \rho} [\phi(\x;\theta)] \Big)^2 \Big]\\
&\leq \E_{\Theta} \Big[ \frac{1}{m^2} \sup_{\|\x\|\leq R} \sum_{j=1}^m \big( \phi(\x;\theta_j) - \E_{\theta'} [\phi(\x;\theta')] \big)^2 \Big]\\
&= \E_{\theta \sim \rho} \Big[ \frac{1}{m} \sup_{\|\x\|\leq R} \big( \phi(\x;\theta) - \E_{\theta'} [\phi(\x;\theta')] \big)^2 \Big]\\
&\leq \E_{\theta \sim \rho} \Big[ \frac{1}{m} \sup_{\|\x\|\leq R} \phi(\x;\theta)^2 \Big]\\
&\leq \E_{\theta \sim \rho} \Big[ \frac{1}{m} \sup_{\|\x\|\leq R} a^2 \|\sigma\|^2_{Lip} (\|\w\|^2+b^2)(\|\x\|^2+1) \Big]\\
&\leq \frac{1}{m} \|\rho\|_P^2 (R^2+1) \|\sigma\|^2_{Lip}\\
&\leq \frac{1}{m} (1+\epsilon)^2 \|V\|^2_{\mathcal{B}} (R^2+1) \|\sigma\|^2_{Lip}
\end{align*}

Meanwhile, denote the empirical measure on $\Theta=(\theta_j)$ by $\rho^{(m)} = \frac{1}{m} \sum_{j=1}^m \delta_{\theta_j}$. Then, its expected path norm is bounded by
\begin{align*}
\E_{\Theta}\big[ \|\rho^{(m)}\|^2_P \big] &= \frac{1}{m} \sum_{j=1}^m \E_{\theta_j} \big[ a_j^2 (\|\w_j\|^2 + b_j^2) \big] = \|\rho\|^2_P \leq (1+\epsilon)^2 \|V\|^2_{\mathcal{B}}
\end{align*}

Define the events
\begin{align*}
E_1 &:= \Big\{ \Theta ~\big|~ \|V-V_m\|^2_{L^{\infty}(P)} \leq 3 \cdot \frac{1}{m} \|V\|^2_{\mathcal{B}} (R^2+1) \|\sigma\|^2_{Lip} \Big\}\\
E_2 &:= \Big\{ \Theta ~\big|~ \|\rho^{(m)}\|^2_P \leq 2 \|V\|^2_{\mathcal{B}} \Big\}
\end{align*}
By Markov's inequality,
\begin{align*}
\mathbb{P}(E_1) &= 1-\mathbb{P}(E_1^C) \geq 1-\frac{\E\big[\|V-V_m\|^2_{L^{\infty}(P)}\big]}{\frac{3}{m} \|V\|^2_{\mathcal{B}} (R^2+1) \|\sigma\|^2_{Lip}} \geq 1-\frac{(1+\epsilon)^2}{3}\\
\mathbb{P}(E_2) &= 1-\mathbb{P}(E_2^C) \geq 1-\frac{\E\big[\|\rho^{(m)}\|^2_P\big]}{2 \|V\|^2_{\mathcal{B}}} \geq 1-\frac{(1+\epsilon)^2}{2}
\end{align*}
Since $\epsilon \in (0,0.01)$,
\begin{equation*}
\mathbb{P}(E_1\cap E_2) = \mathbb{P}(E_1) + \mathbb{P}(E_2) - 1 \geq \frac{1-10\epsilon-5\epsilon^2}{6} > 0
\end{equation*}
Hence, there exists $\Theta = (\theta_j)_{j=1}^m$ such that
\begin{align*}
KL(Q\|Q_m) &\leq 2 \| V_m-V\|_{L^{\infty}(P)} \leq \frac{2\sqrt{3} \|V\|_{\mathcal{B}}}{\sqrt{m}} \|\sigma\|_{Lip} \sqrt{R^2+1}\\
\|V_m\|_{\mathcal{B}} &\leq \|\rho^{(m)}\|_P \leq \sqrt{2} \|V\|_{\mathcal{B}}
\end{align*}
The argument for the case $\|V\|_{\mathcal{H}}<\infty$ is the same.
\end{proof}

\subsection{Proof of Trainability}
\label{appendix. training rule potential}

\begin{lemma}
\label{lemma. convexity of potential objectives}
The objectives $L^+,L^-$ from (\ref{potential model objectives}) are convex in $V$.
\end{lemma}
\begin{proof}
It suffices to show that $\log\E_P[e^V]$ is convex: Given any two potential functions $V_1,V_2$ and any $t \in (0,1)$, H\"{o}lder's inequality implies that
\begin{align*}
\log\E_P\big[e^{tV_1+(1-t)V_2}\big] &= \log\E_P\big[(e^{V_1})^t (e^{V_2})^{(1-t)}\big]\\
&\leq \log \Big( \big\|(e^{V_1})^t\big\|_{L^{1/t}(P)} \big\|(e^{V_2})^{(1-t)}\big\|_{L^{1/(1-t)}(P)} \Big)\\
&= \log \big(\E_P[e^{V_1}]^t \E_P[e^{V_2}]^{(1-t)}\big)\\
&= t \log \E_P[e^{V_1}] + (1-t) \log \E_P[e^{V_2}]
\end{align*}
\end{proof}

\begin{proof}[Proof of Proposition \ref{prop. bias potential RFM}]
For the target potential function $V_*$, denote its parameter function by $a_* \in L^2(\rho_0)$.
Let the objective $L$ be either $L^+$ or $L^-$. The mapping
$$a\mapsto V=\E_{\rho_0}[a(\w,b) \sigma(\w\cdot+b)]$$
is linear while $L$ is convex in $V$ by Lemma \ref{lemma. convexity of potential objectives}, so $L$ is convex in $a\in L^2(\rho_0)$ and we simply write the objective as $L(a)$. Define the Lyapunov function
\begin{equation*}
E(t) = t~\big(L(a_t)-L(a_*)\big) + \frac{1}{2} ||a_*-a_t||^2_{L^2(d\rho_0)}
\end{equation*}
Then,
\begin{align*}
\frac{d}{dt} E(t) &= \big(L(a_t)-L(a_*)\big) + t \cdot \frac{d}{dt} L(a_t) + \big\lb a_t-a_*, ~\frac{d}{dt} a_t \big\rb_{L^2(\rho_0)}\\
&\leq \big(L(a_t)-L(a_*)\big) - \big\lb a_t-a_*, ~\nabla L(a_t) \big\rb_{L^2(\rho_0)}
\end{align*}
By convexity, for any $a_1,a_2$,
\begin{equation*}
L(a_1) + \lb a_2-a_1, ~\nabla L(a_1) \rb \leq L(a_2)
\end{equation*}
Hence, $\frac{d}{dt}E \leq 0$. We conclude that $E(t) \leq E(0)$ or equivalently
\begin{align*}
t~\big(L(a_t)-L(a_*)\big) + \frac{1}{2} ||a_*-a_t||^2_{L^2(d\rho_0)} &\leq \frac{1}{2} ||a^*-a_0||^2_{L^2(d\rho_0)}
\end{align*}

\end{proof}

\begin{assumption}
\label{assumption. potential activation}
We make the following assumptions on the activation function $\sigma(\w\cdot\x+b)$, the initialization $\rho_0$ of $\rho_t$, and the base distribution $P$:
\begin{enumerate}
\item The weights $(\w,b)$ are restricted to the sphere $\mathbb{S}^d \subseteq \R^{d+1}$.
\item The activation is universal in the sense that for any distributions $P,Q$,
$$P=Q \text{ iff } \forall (\w,b)\in\mathbb{S}^d, ~\E_{P-Q}\big[\sigma(\w\cdot\x+b)\big] = 0$$
\item $\sigma$ is continuously differentiable with a Lipschitz derivative $\sigma'$. (For instance, $\sigma$ might  be sigmoid or mollified ReLU.)
\item The initialization $\rho_0 = \rho_0(a,\w,b) \in \PS(\R\times\mathbb{S}^d)$ has full support over $\mathbb{S}^d$. Specifically, the support of $\rho_0$ contains a submanifold that separates the two components, $(\infty,-\overline{a})\times \mathbb{S}^d$ and $(\overline{a}, \infty)\times \mathbb{S}^d$, for some $\overline{a}$.
\item $P$ is compactly-supported.
\end{enumerate}
\end{assumption}

\begin{proof}[Proof of Proposition \ref{prop. bias potential 2-NN}]
The proof follows the arguments of \cite{chizat2018global,rotskoff2019trainability}.
For convenience, denote $(\x,1)$ by $\tilde{\x}$ and $(\w,b)$ by $\w$, so the activation is simply $\sigma(\w\cdot\tilde{\x})$. Denote the training objective by $L$ ($L=L^+$ or $L=L^-$).
From a particle perspective, the flow (\ref{potential 2-NN gradient flow}) can be written as
\begin{align}
\label{potential 2-NN velocity field}
\begin{split}
\dot{a}_t &= -\E_{\Delta_t} \big[ \sigma(\w_t\cdot\tilde{\x}) \big] \\
\dot{\w}_t &= -a_t ~\E_{\Delta_t} \big[ \sigma'(\w_t\cdot\tilde{\x}) ~\tilde{\x} \big]
\end{split}
\end{align}
where $\Delta_t=P_*-P$ if $L=L^+$ and $\Delta_t=Q_*-Q^-_t$ if $L=L^-$.

Since the velocity field (\ref{potential 2-NN velocity field}) is locally Lipschitz over $\R\times\mathbb{S}^d$, the induced flow is a family of locally Lipschitz diffeomorphisms, and thus preserve the submanifold given by Assumption \ref{assumption. potential activation}. Denote by $\hat{\rho}_t$ and $\hat{\rho}_{\infty}$ the projections of $\rho_t, \rho_{\infty}$ onto $\mathbb{S}^d$. It follows that $\hat{\rho}_t$ has full support over $\mathbb{S}^d$ for all time $t<\infty$.

Since $\rho_{\infty}$ is a stationary point of $L$, the velocity field (\ref{potential 2-NN velocity field}) vanishes  at $\rho_{\infty}$ almost everywhere. In particular, for all $\w$ in the support of $\hat{\rho}_{\infty}$,
\begin{equation*}
g(\w) := \E_{\Delta_{\infty}}\big[\sigma(\w\cdot\tilde{\x})\big] = 0
\end{equation*}
We show that this condition holds for all $\w\in \mathbb{S}^d$. Denote $S = \mathbb{S}^d-\sprt\hat{\rho}_{\infty}$. Assume to the contrary  that $g(\w)$ does not vanish on $S$. Let $\w_* \in S$ be a maximizer of $|g(\w)|$. Without loss of generality, let $g(\w_*)>0$; the same reasoning applies to $g(\w_*)<0$.

Since $\rho_t \to \rho_{\infty}$ in $W_1$, the bias potential $V_t$ converges to $V^*$ uniformly over the compact support of $P$. Since all $\Delta_t$ are supported on $\sprt P$, the velocity field (\ref{potential 2-NN velocity field}) converges locally uniformly to
\begin{equation*}
\begin{bmatrix} -\E_{\Delta_{\infty}} [ \sigma(\w\cdot\tilde{\x}) ]\\
-a ~\E_{\Delta_{\infty}} [\sigma'(\w\cdot\tilde{\x}) ~\tilde{\x} ]
\end{bmatrix}
= \begin{bmatrix} -g(\w)\\
-a g'(\w)
\end{bmatrix}
\end{equation*}
For $t$ sufficiently large, we can study the flow with this approximate field. Let $(a,\w)$ be any point with $\w$ sufficiently close to $\w_*$, consider a trajectory $(a_t,\w_t)$ initialized from $a_{t_0}=a, \w_{t_0}=\w$ with a large $t_0$. If $a<0$, then $a_t$ becomes increasingly negative, while $\w_t$ follows a gradient ascent on $g$ and converges to $\w_*$ (or any maximizer nearby). Else, $a\geq 0$, but if $\w$ is sufficiently close to $\w_*$, then $\dot{\w}_t = O(g'(\w))$ is very small (since $g'(\w_*) = 0$ and $g'$ is Lipschitz in $\w$), so $\w_t$ will stay around $\w_*$ and $g(\w_t)$ remains positive. Then, $a_t$ eventually becomes negative, and $\w_t$ converges to $\w_*$.

Since $\hat{\rho}_t$ has positive mass in any neighborhood of $\w_*$ at time $t_0$, this mass will remain in $S$ as $t\to\infty$. This is a contradiction since $S$ is disjoint from $\sprt  \hat{\rho}^*$. It follows that $g(\w)$ vanishes on all of $\mathbb{S}^d$. Then for any $\w \in \mathbb{S}^d$,
\begin{equation*}
\E_{\Delta_{\infty}} \big[ \sigma(\w\cdot \tilde{\x}) \big] = 0
\end{equation*}
By Assumption \ref{assumption. potential activation}, we conclude that $\Delta_{\infty} = 0$, or equivalently $Q_{\infty} = Q_*$ and $V_{\infty} = V_*$ (up to an additive constant).
\end{proof}

\subsection{Proof of Generalization Ability}
\label{appendix. generalization potential}

\begin{proof}[Proof of Lemma \ref{lemma. gradient Monte Carlo rate}]
Theorem 6 of \cite{e2019barron} implies that given any $n$ points with $l^{\infty}$ norm $\leq 1$, the Rademacher complexity of the class $\{\sigma(\w\cdot\tilde{\x}), \|\w\|_1 \leq 1\}$ is bounded by
\begin{equation*}
Rad_n \leq 2\sqrt{\frac{2\log 2d}{n}}
\end{equation*}
Since $|\sigma(\w\cdot\tilde{\x})| \leq 1$ for all $\|\w\|_1\leq 1, \|\x\|_{\infty} \leq 1$, we can apply Theorem 26.5 of \cite{shalev2014understanding} to conclude that
\begin{equation*}
\forall  \|\w\| \leq 1, ~\E_{Q_*-Q_*^{(n)}}[\sigma(\w\cdot\tilde{\x})] \leq 2Rad_n + \sqrt{\frac{2\log (2/\delta)}{n}}
\end{equation*}
with probability $1-\delta$ over the sampling of $Q_*^{(n)}$.
\end{proof}

\begin{proof}[Proof of Lemma \ref{lemma. perturbed convex gradient flow}]
Denote the inner product and norm of $H$ by $\lb x, y \rb$ and $\|x\|$.
Then,
\begin{equation*}
\frac{d}{dt} \|y_t-x_t\| \leq - \big\lb \frac{y_t-x_t}{\|y_t-x_t\|}, ~\nabla L(y_t) - \nabla L(x_t) + \nabla h(y_t) \big\rb
\end{equation*}
Since $L$ is convex, $(y-x)\cdot (\nabla L(y)- \nabla L(x)) \geq 0$ for any $x,y\in H$. Therefore,
\begin{align*}
\frac{d}{dt} \|y_t-x_t\| &\leq - \lb \frac{y_t-x_t}{\|y_t-x_t\|}, \nabla h(y_t) \rb\\
&\leq \| \nabla h(y_t) \| \leq \epsilon
\end{align*}
so that $\|y_t-x_t\| \leq \epsilon t$. By Lipschitz continuity, $L(y_t) - L(x_t) \leq l\epsilon t$.
\end{proof}

\begin{proof}[Proof of Theorem \ref{thm. generalization for potential model}]
For any time $T$, the testing error can be decomposed into
\begin{align*}
KL\big(Q_*\|Q_T^{(n)}\big) &= L(V_T^{(n)}) - L(V_*)\\
&= \big(L(V_T^{(n)})-L(V_T)\big) + \big(L(V_T)-L(V_*)\big)
\end{align*}
The second term is bounded by Proposition \ref{prop. bias potential RFM}, while the first term can be bounded by Lemma \ref{lemma. perturbed convex gradient flow}. The Hilbert space $H$ in Lemma \ref{lemma. perturbed convex gradient flow} corresponds to the parameter functions $L^2(\rho_0)$ for the random feature model, the convex objective corresponds to the objective $L$ over $a \in L^2(\rho_0)$,
\begin{equation*}
L(a) = \E_{Q_*}[V] + \log\E_{P}[e^{-V}], \quad V(\x)=\E_{\rho_0(\w)}[a(\w)\sigma(\w\cdot\tilde{\x})]
\end{equation*}
and the perturbation term $h$ corresponds to $L^{(n)}-L$,
\begin{equation*}
L^{(n)}(a) - L(a) = \E_{Q^{(n)}_*-Q_*}[V]
\end{equation*}
The remaining task is to estimate the constants $l$ and $\epsilon$.

First, we have $l\leq 2$. For any $a\in L^2(\rho_0)$, let $Q$ be the modeled distribution,
\begin{align*}
\|\nabla L(a)\|_{L^2(\rho_0)} &= \|\E_{Q_*-Q}[\sigma(\w\cdot\tilde{\x})]\|_{L^2(d\rho_0(\w))}\\
&\leq \sup_{\|\w\|_1\leq 1} |\E_{Q_*-Q}[\sigma(\w\cdot\tilde{\x})]|\\
&\leq \sup_{\|\w\|_1\leq 1} |\E_{Q_*}[\sigma(\w\cdot\tilde{\x})]| + \sup_{\|\w\|_1\leq 1} |\E_{Q}[\sigma(\w\cdot\tilde{\x})]|\\
&\leq 2
\end{align*}
where in the last step, since all distributions are supported on $[-1,1]^d$, $\sigma(\w\cdot\tilde{\x}) \leq \|\w\|_1 \|\tilde{\x}\|_{\infty} \leq 1$.

Next, the estimate of $\epsilon$ has been provided by Lemma \ref{lemma. gradient Monte Carlo rate}, because for any $a \in L^2(\rho_0)$,
\begin{align*}
\|\nabla h(a)\|_{L^2(\rho_0)} &= \|\nabla L^{(n)}(a) - \nabla L(a)\|_{L^2}\\
&= \|\E_{Q_*-Q_*^{(n)}}[\sigma(\w\cdot\tilde{\x})] \|_{L^2} \\
&\leq \|\E_{Q_*-Q_*^{(n)}}[\sigma(\w\cdot\tilde{\x})] \|_{L^{\infty}(\rho_0)}\\
&\leq \sup_{\|\w\|_1 \leq 1} |\E_{Q_*-Q_*^{(n)}}[\sigma(\w\cdot\tilde{\x})]|
\end{align*}

\end{proof}

\subsection{Proof of Memorization}
\label{appendix: memorization}

To prove Proposition \ref{prop. memorization} and Lemma \ref{lemma. universal convergence}, we begin with a few useful lemmas.

Let $\mathcal{M}(K)$ be the space of finite signed measures on $K$. We say that a kernel $k$ is integrally strictly positive definite if
\begin{equation*}
\forall m \in \mathcal{M}(K), ~\E_{m(\x)}\E_{m(\x')}[k(\x,\x')] \to m=0
\end{equation*}
Equip $\mathcal{M}(K)$ with the inner product
\begin{equation*}
\forall m_1,m_2 \in \mathcal{M}(K), ~\lb m_1,m_2\rb_k := \E_{m_1(\x)}\E_{m_2(\x')}[k(\x,\x')]
\end{equation*}
from which we define the MMD (maximum mean discrepancy) distance $\|\cdot\|_k$
\begin{equation*}
\|m_1-m_2\|_k^2 = \lb m_1-m_2,m_1-m_2 \rb_k
\end{equation*}
Let $\H_k$ be the RKHS generated by $k$ with inner product $\lb,\rb_{\H_k}$. Then the MMD inner product is the RKHS inner product on the mean embeddings $f_i=\E_{m_i(\x)}[k(\x,\cdot)]$,
\begin{align*}
\lb m_1,m_2 \rb_k &= \lb f_1,f_2 \rb_{\H_k}\\
\|m_1-m_2\|_k &= \sup_{\|f\|_{\H_k} \leq 1} \E_{m_1-m_2}[f]
\end{align*}

\begin{lemma}
\label{lemma. MMD weak topology}
When restricted to the subset $\PS(K)$, the MMD distance $\|\cdot\|_k$ induces the weak topology
and thus $(\PS(K),\|\cdot\|_k)$ is compact.
\end{lemma}
\begin{proof}
By Lemma 2.1 of \cite{simongabriel2020metrizing}, the MMD distance metrizes the weak topology of $\PS(K)$, which is compact by Prokhorov's theorem.
\end{proof}

As $\PS(K)$ is a convex subset of $\mathcal{M}(K)$, we can define the tangent cone at each point $Q\in\PS(K)$ by
\begin{equation*}
T_Q \PS(K) = \big\{ \lambda \Delta ~\big|~ \lambda \geq 0, ~\Delta = \Delta^+-\Delta^-, ~\Delta^{\pm}\in\PS(K), ~\Delta^- \ll Q \big\}
\end{equation*}
and equip it with the MMD norm, $\|\Delta\|_k^2 = \E_{\Delta^2}[k]$.\\

Given the gradient flow $V_t$ defined in Lemma \ref{lemma. universal convergence}, the distribution $Q_t$ evolves by
\begin{align*}
\frac{d}{dt} Q_t(\x) &= \big(v(\x;Q_t)-\E_{Q_t(\x')}[v(\x';Q_t)]\big) Q_t(\x)\\
v(\x;Q) &:= \E_{(Q'-Q)(\x')}[k(\x,\x')]
\end{align*}
We can extend this flow to a dynamical system on $\PS(K)$ in positive time $t \geq 0$, defined by
\begin{align}
\label{measure dynamical system}
\begin{split}
\frac{d}{dt} Q_t &= \overline{v}(Q_t)Q_t \\
\overline{v}(Q) &= v(\cdot~;Q)-\E_{Q(\x')}[v(\x';Q)]
\end{split}
\end{align}
Each $\overline{v}(Q)Q$ is a tangent vector in $T_Q\mathcal{P}(K)$.

Note that we can rewrite $v$ and $\overline{v}$ in terms of the RKHS norm: Let $f,f'$ be the mean embeddings of $Q,Q'$,
\begin{align*}
v(\x;Q) &= \big\lb k(\x,\cdot), ~f'-f \big\rb_{\H_k}\\
\overline{v}(\x;Q) &= \big\lb k(\x,\cdot) - f, ~f'-f \big\rb_{\H_k}
\end{align*}
It follows that $v$ and $\overline{v}$ are uniformly continuous over the compact space $K \times (\PS(K),\|\cdot\|_k)$.

\begin{lemma}
\label{lemma. ODE solution}
Given any initialization $Q_0 \in \PS(K)$, there exists a unique solution $Q_t$, $t\geq 0$ to the dynamics (\ref{measure dynamical system}).
\end{lemma}
\begin{proof}
The integral form of (\ref{measure dynamical system}) can be written as
\begin{equation}
\label{integral measure dynamical system}
\forall t \geq 0, ~Q_t = Q_0 + \int_0^t \overline{v}(Q_s) Q_s ds
\end{equation}
where we adopt the Bochner integral on $(\mathcal{M}(K),\lb,\rb_k)$. In the spirit of the classical Picard-Lindel\"{o}f theorem, we consider the vector space $C([0,T], \mathcal{M}(K))$ equipped with sup-norm
\begin{equation*}
|||\phi||| = \sup_{t\in[0,T]} \|\phi(t)\|_k
\end{equation*}
On the complete subspace $C([0,T], \PS(K))$, define the operator $\phi \mapsto F(\phi)$ by
\begin{equation*}
F(\phi)_t = \phi_0 + \int_0^t \overline{v}(\phi_s) \phi_s ds
\end{equation*}
Define the sequence $\phi^0 \equiv Q_0$ and $\phi^{n+1} = F(\phi^n)$.

Note that the tangent field (\ref{measure dynamical system}) is Lipschitz
\begin{equation*}
\forall Q_1,Q_2, \quad \|\overline{v}(Q_1)Q_1-\overline{v}(Q_2)Q_2\|_k \leq c \|Q_1-Q_2\|_k
\end{equation*}
with $c \leq 4(\|k\|_{C(K\times K)}^2+\|k\|_{C(K\times K)})$.
Then, with $T \leq 1/2c$,
\begin{align*}
|||\phi^{n+1}-\phi^n||| &\leq \sup_{t\in[0,T]} \big\| \int_0^t \overline{v}(\phi^n_s)\phi^n_s - \overline{v}(\phi^{n-1}_s)\phi^{n-1}_s ds \big\|_k\\
&\leq \int_0^T \| \overline{v}(\phi^n_t)\phi^n_t - \overline{v}(\phi^{n-1}_t)\phi^{n-1}_t \|_k dt\\
&\leq cT \sup_{t\in[0,T]} \|\phi^n_t-\phi^{n-1}_t\|_k\\
&\leq \frac{1}{2} |||\phi^n-\phi^{n-1}|||
\end{align*}
By the completeness of $(C([0,T], \PS(K)),|||\cdot|||)$ and Banach fixed point theorem, the sequence $\phi^n$ converges to a unique solution $\phi$ of (\ref{integral measure dynamical system}) on $[0,T]$. Then, we can extend this solution iteratively to $[T,2T],[2T,3T],\dots$ and obtain a unique solution on $[0,\infty)$.
\end{proof}

Denote the set of fixed points of (\ref{measure dynamical system}) by
\begin{equation*}
\PS_o = \{Q\in\PS(K) ~|~ \overline{v}(Q)Q = 0\}
\end{equation*}
Also, define the set of distributions that have larger supports than the target distribution $Q'$
\begin{equation*}
\PS_* = \{Q\in\PS(K) ~|~ \text{sprt} Q' \subseteq \text{sprt} Q\}
\end{equation*}

\vspace{1em}
\begin{lemma}
\label{lemma. measure dynamics fixed points}
We have the following inclusion
\begin{equation*}
\PS_o \subseteq \{Q'\} \cup (\PS(K)-\PS_*)
\end{equation*}
Given any initialization $Q_0 \in \PS(K)$, let $Q_t$, $t\geq 0$ be the trajectory defined by Lemma \ref{lemma. ODE solution} and let $\mathcal{Q}$ be the set of limit points in MMD metric
\begin{equation*}
\mathcal{Q} = \bigcap_{T\to\infty} \overline{\{Q_t, t\geq T\}}^{\|\cdot\|_k}
\end{equation*}
then $\mathcal{Q} \subseteq \PS_o$.
\end{lemma}

\begin{proof}
For any fixed point $Q \in \PS_o$, we have $\overline{v}(\x;Q)=0$ for $Q$-almost every $\x$.
By continuity, we have
\begin{equation}
\label{fixed point constant velocity}
\forall \x \in \sprt Q, \quad v(\x;Q) = \E_{Q(\x')}[v(\x';Q)]
\end{equation}
If we further suppose that $Q \in \PS_*$, then this equality holds for $Q'$-almost all $\x$, so
\begin{align*}
0 &= \E_{(Q-Q')(\x)}[v(\x;Q)]\\
&= \E_{(Q-Q')^2(\x,\x')}[k(\x,\x')]\\
&= \|Q-Q'\|_k^2
\end{align*}
Since $k$ is integrally strictly positive definite, we have $Q=Q'$.
It follows that
\begin{equation*}
\PS_o \cap \PS_* = \{Q'\}
\end{equation*}
or equivalently $\PS_o \subseteq \{Q'\} \cup (\PS(K)-\PS_*)$.

Meanwhile, the MMD distance $\|Q_t-Q'\|^2_k$ is decreasing along any trajectory $Q_t$ of (\ref{measure dynamical system}):
\begin{align}
\label{MMD Lyapunov}
\begin{split}
\frac{d}{dt} \frac{1}{2} \|Q_t-Q'\|^2_k &= \E_{Q_t(\x)} \E_{(Q_t-Q')(\x')} \big[ k(\x,\x') ~\overline{v}(\x;Q_t) \big]\\
&= -\E_{Q_t(\x)} \big[\overline{v}(\x;Q_t)^2 \big]\\
&\leq 0
\end{split}
\end{align}
Define the extended sublevel sets for every $c>0$,
\begin{equation*}
\PS_c := \{Q\in\PS(K)~|~\|Q-Q'\|_k \leq c \text{ or } Q\in\PS_o\}
\end{equation*}
By Lemma \ref{lemma. MMD weak topology}, the space $(\PS(K),\|\cdot\|_k)$ is compact, so the set of limit points $\mathcal{Q}$ of the trajectory $Q_t$ is nonempty.
The inequality (\ref{MMD Lyapunov}) is strict if $Q_t \notin \PS_o$, so these limit points all belong to
\begin{equation*}
\bigcap_{c\to 0^+} \PS_c = \PS_o
\end{equation*}
\end{proof}

\begin{lemma}
\label{lemma. measure dynamics global convergence}
Given any initialization $Q_0 \in \PS_*$, if the limit point set $\mathcal{Q}$ contains only one point $Q_{\infty}$, then $Q_{\infty} \in \PS_*$ and thus $Q_{\infty} = Q'$.
Else, $\mathcal{Q}$ is contained in $\PS(K)-\PS_*$.
\end{lemma}
\begin{proof}
For any open subset $A$ that intersects $\sprt Q_0$, we have $Q_0(A)>0$. Also
\begin{align*}
\frac{d}{dt} Q_t(A) &= \E_{Q_t}[\mathbf{1}_A(\x) \overline{v}(\x;Q_t)]\\
&\geq -Q_t(A) \|\overline{v}(Q_t)\|_{L^{\infty}(Q_t)}\\
&\geq -4\|k\|_{C(K\times K)} Q_t(A)
\end{align*}
So $Q_t(A)$ remains positive for all finite $t$. It follows that $\sprt Q_0 \subseteq \sprt Q_t$ and $Q_t \in \PS_*$ for all $t$.

First, consider the case $\mathcal{Q}=\{Q_{\infty}\}$. Assume for contradiction that $Q_{\infty} = \tilde{Q}$ for some $\tilde{Q} \in \PS_o-\{Q'\} \subseteq \PS(K)-\PS_*$. Equation (\ref{fixed point constant velocity}) implies that
\begin{equation*}
\E_{\tilde{Q}} [\overline{v}(\x;\tilde{Q})] = 0
\end{equation*}
and thus
\begin{align*}
\E_{Q'}[\overline{v}(\x;\tilde{Q})] &= \E_{Q'-\tilde{Q}}[\overline{v}(\x;\tilde{Q})]\\
&= \E_{Q'-\tilde{Q}}[v(\x;\tilde{Q})]\\
&= \|Q'-\tilde{Q}\|_k^2\\
&>0
\end{align*}
In particular, there exists some measureable subset $S_o \subseteq \sprt Q'$ and some $\delta>0$ such that
\begin{equation*}
\forall \x \in S_o, ~\overline{v}(\x;\tilde{Q}) > 2\delta
\end{equation*}
By continuity, there exists some open subset $S$ ($S_o\subseteq S$) such that its closure $\overline{S}$ satisfies
\begin{equation*}
\forall \x \in \overline{S}, ~\overline{v}(\x;\tilde{Q}) \geq \delta
\end{equation*}
Meanwhile, since $S$ intersects $\sprt Q' \subseteq \sprt Q_t$, we have $Q_t(S) > 0$ for all $t$. Whereas (\ref{fixed point constant velocity}) implies that $\overline{S}$ is disjoint from $\sprt \tilde{Q}$.

Since $\overline{v}$ is continuous over $(\x,Q) \in K \times (\PS(K),\|\cdot\|_k)$ and $\overline{S}$ is compact, there exists some neighborhood $B_r(\tilde{Q}) = \{Q\in\PS(K)~|~\|Q-\tilde{Q}\|_k < r\}$ such that
\begin{equation*}
\forall Q\in B_r(\tilde{Q}),~ \forall \x \in \overline{S}, ~\overline{v}(\x;\tilde{Q}) \geq 0
\end{equation*}

Since the trajectory $Q_t$ converges in the MMD distance $\|\cdot\|_k$ to $\tilde{Q}$, there exists some time $t_0$ such that for all $t \geq t_0$, $Q_t \in B_r(\tilde{Q})$. It follows that
\begin{equation*}
\frac{d}{dt} Q_t(\overline{S}) = \E_{Q_t}[\mathbf{1}_{\overline{S}}(\x) \overline{v}(\x;Q_t)] \geq 0
\end{equation*}
so that $Q_t(\overline{S}) \geq Q_{t_0}(\overline{S})$ for all $t\geq t_0$. Yet, Lemma \ref{lemma. MMD weak topology} implies that $Q_t$ converges weakly to $\tilde{Q}$, so that
\begin{equation*}
0 = \tilde{Q}(\overline{S}) \geq \limsup_{t \to \infty} Q_t(\overline{S})
\end{equation*}
A contradiction. We conclude that the limit point $Q_{\infty}$ does not belong to $\PS_o-\{Q'\}$. By Lemma \ref{lemma. measure dynamics fixed points}, we must have $Q_{\infty} = Q'$.

Next, consider the case when $\mathcal{Q}$ has more than one point.
Inequality (\ref{MMD Lyapunov}) implies that the MMD distance $L(Q) = \|Q-Q_*^{(n)}\|^2_k$ is monotonously decreasing along the flow $Q_t$. Suppose that $Q' \in \mathcal{Q}$, then $\lim_{t\to\infty} L(Q_t) = 0$ and thus $\mathcal{Q} = \{Q'\}$, a contradiction. Hence, $\mathcal{Q} \subseteq \PS_o-\{Q'\} \subseteq \PS(K) - \PS_*$.
\end{proof}

\begin{proof}[Proof of Lemma \ref{lemma. universal convergence}]
Since $V_0 \in C(K)$, the initialization $Q_0$ has full support over $K$ and thus $Q_0 \in \PS_*$.
If $Q_t$ converges weakly to some limit $Q_{\infty}$, Lemma \ref{lemma. MMD weak topology} implies that $Q_t$ also converges in MMD metric to $Q_{\infty}$.
Then, Lemma \ref{lemma. measure dynamics global convergence} implies that the limit $Q_{\infty}$ must be $Q'$.

If there are more than one limit, then Lemma \ref{lemma. measure dynamics global convergence} implies that all limit points belong to $\PS(K)-\PS_*$ and thus do not cover the full support of $Q'$.
\end{proof}

\begin{proof}[Proof of Proposition \ref{prop. memorization}]
We simply set $Q'=Q_*^{(n)}$. Note that since $a_t^{(n)}$ is trained by
\begin{equation*}
\frac{d}{dt} a_t^{(n)}(\w) = \E_{(Q_t^{(n)}-Q_*^{(n)})(\x)}[\sigma(\w\cdot\tilde{\x})]
\end{equation*}
the training dynamics for the potential $V_t^{(n)}$ is the same as in Lemma \ref{lemma. universal convergence}
\begin{equation*}
\frac{d}{dt} V^{(n)}_t(\x) = \E_{(Q_t^{(n)}-Q_*^{(n)})(\x')}[k(\x,\x')]
\end{equation*}
with kernel $k$ defined by
\begin{equation*}
k(\x,\x') = \E_{\rho_0(\w)}[\sigma(\w\cdot\tilde{\x})\sigma(\w\cdot\tilde{\x}')]
\end{equation*}
It is straightforward to check that $k$ is integrally strictly positive definite: For any $m \in \mathcal{M}(K)$, if
\begin{equation*}
0 = \|m\|_k^2 = \E_{\rho_0(\w)}\big(\E_{m(\x)}[\sigma(\w\cdot\tilde{\x}))]\big)^2
\end{equation*}
then for $\rho_0$-almost all $\w$, $\E_{m(\x)}[\sigma(\w\cdot\tilde{\x})] = 0$. It follows that for all random feature models $f$ from (\ref{continuous RFM}), we have $\E_{m(\x)}[f(\x)] = 0$.
Assuming Remark \ref{remark. universal approximation}, the random feature models are dense in $C(K)$ by Proposition \ref{prop. universal approximation potential}, so this equality holds for all $f \in C(K)$. Hence, $m=0$ and $k$ is integrally strictly positive definite.

Hence, Lemma \ref{lemma. universal convergence} implies that if $Q_t^{(n)}$ has one limit point, then $Q_t^{(n)}$ converges weakly to $Q_*^{(n)}$. Else, no limit point can cover the support of $Q_*^{(n)}$ and thus do not have full support over $K$.
Since the true target distribution $Q_*$ is generated by a continuous potential $V_*$, it has full support and thus does not belong to $\mathcal{Q}$ and $KL(Q_*\|Q)=\infty$ for all $Q\in \mathcal{Q}$.
Similarly, we must have
\begin{equation*}
\liminf_{t\to\infty} \|V_t^{(n)}\|_{\H} = \infty
\end{equation*}
otherwise some subsequence of $Q_t^{(n)}$ would converge to a limit with full support.
\end{proof}

\subsection{Proof for the Regularized Model}

\begin{lemma}
\label{lemma. existence regularized minimizer}
For any $R\geq 0$, there exists a minimizer of (\ref{constrained empirical loss}).
\end{lemma}

\begin{proof}
Since the closed ball $B_R = \{\|a\|_{L^2(\rho_0)}\leq R\}$ is weakly compact in $L^2(\rho_0)$, it suffices to show that the mapping
\begin{equation*}
L^{(n)}(a) = \E_{\rho_0(\w)}\big[ a(\w) \E_{Q_*^{(n)}(\x)}[\sigma(\w\cdot\tilde{\x})] \big] + \log \E_{P(\x)}\big[e^{-\E_{\rho_0(\w)}[a(\w)\sigma(\w\cdot\tilde{\x})]}\big]
\end{equation*}
is weakly continuous over $B_R$ (e.g. show that the term $\E_P \big[e^{-\E_{\rho_0(\w)}[a(\w)\sigma(\w\cdot\tilde{\x})]}\big]$ can be expressed as the uniform limit of a sequence of weakly continuous functions over $B_R$).
Then, every minimizing sequence of $L^{(n)}$ in $B_R$ converges weakly to a minimizer of (\ref{constrained empirical loss}).
\end{proof}

\begin{proof}[Proof of Proposition \ref{prop. regularized model}]

For any $a \in L^2(\rho_0)$,
\begin{align*}
|L(a)-L^{(n)}(a)| &\leq \E_{\rho_0(\w)}\big[|\E_{Q_*-Q_*^{(n)}}[a(\w) \sigma(\w\cdot\tilde{\x})]|\big]\\
&\leq \|a\|_{L^2(\rho_0)} \cdot \sup_{\|\w\|_1\leq 1} \E_{Q_*-Q_*^{(n)}}[\sigma(\w\cdot\tilde{\x})]
\end{align*}
Thus, Lemma \ref{lemma. gradient Monte Carlo rate} implies that with probability $1-\delta$ over the sampling of $Q_*^{(n)}$,
\begin{equation}
\label{sampling gap RKHS bound}
|L(a)-L^{(n)}(a)| \leq \|a\|_{L^2(\rho_0)} \cdot \Big( 4\sqrt{\frac{2\log 2d}{n}} + \sqrt{\frac{2\log (2/\delta)}{n}} \Big)
\end{equation}
It follows that
\begin{align*}
L(a^{(n)}_R) &\leq L^{(n)}(a^{(n)}_R) + \frac{4\sqrt{2\log 2d} + \sqrt{2\log (2/\delta)}}{\sqrt{n}} R\\
&\leq L^{(n)}(a_*) + \frac{4\sqrt{2\log 2d} + \sqrt{2\log (2/\delta)}}{\sqrt{n}} R\\
&\leq L(a_*) + \frac{4\sqrt{2\log 2d} + \sqrt{2\log (2/\delta)}}{\sqrt{n}} (R+\|a_*\|_{L^2(\rho_*)})
\end{align*}
where the first and third inequalities follow from (\ref{sampling gap RKHS bound}) and the second inequality follows from the fact that $a_* \in \{\|a\|_{L^2(\rho_0)}\leq R\}$.

Hence,
\begin{equation*}
KL(Q_*\|Q^{(n)}_R) = L(a^{(n)}_R) - L(a_*) \leq 2R \cdot \frac{4\sqrt{2\log 2d} + \sqrt{2\log (2/\delta)}}{\sqrt{n}}
\end{equation*}
\end{proof}

\section{Discussion}
\label{sec. discussions}

Let us summarize some of the insights obtained in this paper:
\begin{itemize}
\item For distribution-learning models, good generalization can be characterized by dimension-independent \textit{a priori} error estimates for early-stopping solutions. As demonstrated by the proof of Theorem \ref{thm. generalization for potential model}, such estimates are enabled by two conditions:
\begin{enumerate}
\item Fast global convergence is guaranteed for learning distributions that can be represented by the model, with an explicit and dimension-independent rate. For our example, this results from the convexity of the model.
\item The model is insensitive to the sampling error $Q_*-Q_*^{(n)}$, so memorization happens very slowly and early-stopping solutions generalize well. For our example, this is enabled by the small Rademacher complexity of the random feature model.
\end{enumerate}

\item Memorization seems inevitable for all sufficiently expressive models (Proposition \ref{prop. memorization}), and the generalization error $\widetilde{L}$ will eventually deteriorate to either $n^{-O(1/d)}$ or $\infty$. Thus, instead of the long time limit $t \to\infty$, one needs to consider  early-stopping.

The basic approach, as suggested by Theorem \ref{thm. generalization for potential model}, is to choose an appropriate function representation such that, with absolute constants $\alpha_1,\alpha_2 > 0$, there exists an early-stopping interval $[T_{\min}, T_{\max}]$ with $T_{\min} \ll n^{\alpha_1} \ll T_{\max}$ and
\begin{equation}
\label{avoid memorization}
\sup_{t \in [T_{\min}, T_{\max}]} \widetilde{L} \big(Q_*, Q_t^{(n)}\big) = O(n^{-\alpha_2})
\end{equation}
Then, with a reasonably large sample set (polynomial in precision $\epsilon^{-1}$), the early-stopping interval will become sufficiently wide and hard to miss, and the corresponding generalization error will be satisfactorily small.



\item A distribution-learning model can be posed as a calculus of variations problem. Given a training objective $L(Q)$ and distribution representation $Q(f)$, this problem is entirely determined by the function representation or function space $\{\|f\|<\infty\}$. Given a training rule, the choice of the function representation then determines the trainability (Proposition \ref{prop. bias potential 2-NN}) and generalization ability (Theorem \ref{thm. generalization for potential model}) of the model.

\end{itemize}

\noindent
Future work can be developed from the above insights:
\begin{itemize}
\item Generalization error estimates for GANs

The Rademacher complexity argument should be applicable to GANs to bound the deviation $\|G_t-G_t^{(n)}\|_{L^2(P)}$, where $G_t,G_t^{(n)}$ are the generators trained on $Q_*$ and $Q_*^{(n)}$ respectively. Nevertheless, the difficulty is in the convergence analysis.
Unlike bias potential models, the training objective of GAN is non-convex in the generator $G$, and the solutions to $G\#P=Q_*$ are in general not unique.  

\item Mode collapse

If we consider mode collapse as a form of bad local minima, then it can benefit from a study of the critical points of GAN, once we pose GAN as a calculus of variations problem.
Unlike the bias potential model whose parameter function $V$ ranges in the Hilbert space $\mathcal{H}$, GANs are formulated on the Wasserstein manifold whose tangent space $L^2(Q;\R^d)$ depends significantly on the current position $Q$. In particular, the behavior of gradient flow differs whether $Q$ is absolutely continuous or not, and we expect that successful GAN models can maintain the
absolutely continuity of the  trajectory $Q_t$.

\item New designs

The design of distribution-learning model can benefit from a mathematical understanding. 
For instance, consider the early-stopping interval (\ref{avoid memorization}), can there be better training rules than gradient flow that reduces $T_{\min}$ or postpones $T_{\max}$ so that early-stopping becomes easier to perform?

\end{itemize}

\bibliography{main}
\bibliographystyle{acm}

\end{document}